\documentclass[preprint,12pt]{elsarticle}

\usepackage{amsmath, amssymb, amsthm}
\usepackage{graphicx}
\usepackage[colorlinks=true, linkcolor=blue, citecolor=red, urlcolor=blue]{hyperref}
\usepackage{tikz}
\usepackage{pgfplots}
\usepackage{enumitem} % For custom lists
\usepackage{xspace}   % Proper spacing after macros
\usepackage{microtype} % Improves typography
\usepackage{xpatch}

\title{A Verifier Hierarchy}
\author[1]{Maurits Kaptein\corref{cor1}}
\ead{m.c.kaptein@tue.nl}

\address[1]{Department of Mathematics and Computer Science, University of Eindhoven, De Zaale 1, 5600MB, Eindhoven, the Netherlands}

\cortext[cor1]{Corresponding author.}

\date{\today}

\newcommand{\np}{\textsf{NP}\xspace}
\newcommand{\p}{\textsf{P}\xspace}
\newcommand{\exptime}{\textsf{EXPTIME}\xspace}
\newcommand{\SAT}{\textsc{3-SAT}\xspace}

\newcommand{\poly}[1][]{\mathit{poly}#1}

\newcommand{\period}{\textsc{PERIODIC}\xspace}
\newcommand{\rotate}{\textsc{STRING-ROTATION}\xspace}

\xpatchcmd{\proof}{\itshape}{\bfseries}{}{}

\theoremstyle{plain}
\newtheorem{theorem}{Theorem}[section]
\newtheorem{lemma}[theorem]{Lemma}
\newtheorem{corollary}[theorem]{Corollary}

\theoremstyle{definition}
\newtheorem{definition}{Definition}[section]
\newtheorem{example}[definition]{Example}
\newtheorem{remark}[definition]{Remark}

\usepackage[toc,page]{appendix}

\fntext[date]{Manuscript prepared on \today.}

\begin{document}

\begin{abstract}
We investigate the trade-off between certificate length and verifier runtime. We prove a Verifier Trade-off Theorem showing that reducing the inherent verification time of a language from \(f(n)\) to \(g(n)\), where \(f(n) \ge g(n)\), requires certificates of length at least \(\Omega(\log(f(n) / g(n)))\). This theorem induces a natural hierarchy based on certificate complexity. We demonstrate its applicability to analyzing conjectured separations between complexity classes (e.g., \(\np\) and \(\exptime\)) and to studying natural problems such as string periodicity and rotation detection. Additionally, we provide perspectives on the \(\p\) vs. \(\np\) problem by relating it to the existence of sub-linear certificates.
\end{abstract}

\begin{keyword}
Computational complexity \sep Certificate complexity \sep Hierarchy theorems \sep Verification paradigms.
\end{keyword}

\maketitle

\section{Introduction}
Hierarchy theorems have long been indispensable tools for separating complexity classes and understanding computational hardness. The deterministic time hierarchy theorem \citep{hartmanis1965computational} established the strict separation of polynomial time \(\p\) from exponential time \(\exptime\), while space hierarchy theorems \citep{stearns1965hierarchies} delineate distinct space complexity classes. Similarly, nondeterministic time hierarchy theorems \citep{cook1972hierarchy, even1982note,fortnow2000time} reveal subtler separations for nondeterministic Turing machines (NTMs). These results underscore the foundational role of hierarchies in complexity theory.

In this work, we introduce a novel hierarchy theorem that bridges deterministic and nondeterministic computation through certificate complexity. Building on the equivalence of NTMs and deterministic verifiers for \(\np\) \citep[as reprinted in][]{cook2023complexity}, we establish a \emph{hierarchy of verifiers} where certificate length governs computational speed-ups. Formally, suppose that for a language \(L\) and two certificate-length functions \(b_1(n)<b_2(n)\) we have:
\begin{itemize}
  \item Every verifier for \(L\) using \(b_1(n)\)-bit certificates requires at least \(\Omega(f(n))\) time (the \emph{inherent} verification time).
  \item There exists a verifier for \(L\) using \(b_2(n)\)-bit certificates that runs in \(O(g(n))\) time,
\end{itemize}
with \(f(n)\ge g(n)\). We prove in Theorem~\ref{thm:tradeoff} that
\[
  b_2(n) - b_1(n) = \Omega\left(\log\frac{f(n)}{g(n)}\right).
\]
This captures the trade-off: reducing a language’s inherent verification time from \(f(n)\) to \(g(n)\) forces an additive certificate-length increase of \(\Omega\left(\log(f/g)\right)\), yielding a strict hierarchy of verifier complexity.

Anchoring our analysis in the verifier-based view of \np, we illustrate three applications of the Verifier Trade-off Hierarchy:
\begin{itemize}
  \item \textbf{Hierarchy Alignments.}  
    We demonstrate how our theorem aligns with classic separations like \(\p \subsetneq \exptime\). For example, any language in \exptime with a polynomial-time verifier requires super-polynomial certificates.
  \item \textbf{Case Studies on Natural Problems.}  
    We analyze two canonical string-manipulation languages, \period and \rotate, deriving bounds on the certificate length vs. verification time trade-off.
  \item \textbf{Certificate-Size Lens on \p vs. \np.}  
    We recast the \p vs. \np question in terms of certificate length:  
    \begin{itemize}
      \item Linear certificates for \np-complete languages align with the classical definition of \np.
      \item Sub-linear certificates would collapse \(\p\) and \(\np\), suggesting \(\p = \np\).  
    \end{itemize}
    We concretize the latter by analyzing \(\SAT\), highlighting certificate size as a critical resource.
\end{itemize}

In the remainder of this paper, we first briefly review related work (Section \ref{sec:related}) and then introduce notation and definitions (Section \ref{ref:preliminaries}). Section \ref{sec:hierarchy} presents our main Verifier Trade-off Hierarchy Theorem and corollaries. In Section \ref{sec:applications} we discuss implications, analyze natural problems, and explore complexity class separations. Section \ref{sec:comments} provides pointers for future work.

\section{Related Work}
\label{sec:related}

Classical hierarchy theorems delineate separations within computational models. For example, the deterministic time hierarchy theorem \citep{hartmanis1965computational} establishes that sufficiently increasing time strictly enlarges computable classes (e.g., \(\p \subsetneq \exptime\)). Analogous hierarchies for NTMs require super-polynomial time gaps, reflecting the subtler structure of nondeterminism \citep{seiferas1978gap}. These results, while foundational, focus on separations within paradigms. Our work bridges across them: we quantify how certificate size mediates the computational advantage of verification over solving, complementing hierarchy theorems with a dimension orthogonal to time/space.  

The definition of \(\np\) via polynomial-time verifiers \citep{cook2023complexity} inspired our certificate-centric approach. Subsequent verification paradigms—interactive proofs \citep{goldwasser2019knowledge}, probabilistically checkable proofs (PCPs) \citep{arora1998proof}, and automatizability of proof systems \citep{Atserias2004Automatizability}—shifted focus to distributed, randomized, or heuristic verification. For instance, \citep{Atserias2004Automatizability} showed that efficient automatizability of resolution proofs would imply unexpected circuit lower bounds, highlighting the interplay between proof length and computational hardness. In contrast, our framework isolates certificate size as the critical resource, rather than interaction or randomness, mirroring classical hierarchy intuitions in a new dimension.  

Non-uniform computation augments deterministic machines with advice strings fixed per input length \citep{karp1982turing,allender2000complexity}, while certificate complexity \citep{cook2023complexity} measures minimal proof lengths for Boolean functions \citep[see, e.g.,][]{chaubal2021diameter}. Our work diverges by treating certificates as efficiency parameters governing possible language verification speed-ups, rather than static advice or worst-case bounds. This aligns with circuit-based approaches like \citep{Kannan1993Toda}'s proof of Toda’s theorem \citep{Kannan1993Toda}, which leverages circuit complexity to ascend the polynomial hierarchy, but with a focus on dynamic proofs rather than fixed circuit families.  

Modern complexity theory emphasizes fine-grained lower bounds under conjectures like the Exponential Time Hypothesis (\(\mathsf{ETH}\)) \citep{impagliazzo2001complexity,chen2020exponential} and tools from fine-grained complexity \cite[e.g.,][]{williams2018fine}. These results often assume hardness to derive precise time constraints (e.g., for \(\SAT\)). Our hierarchy theorem complements this by deriving certificate size lower bounds from structural problem properties. For example, connections between combinatorial games and proofs, as in \citep{Atserias2011MeanPayoff}'s work on mean-payoff games and propositional proofs \citep{Atserias2011MeanPayoff}, reveal how problem structure constrains proof complexity—a theme our work extends to certificate-mediated verification.  

%By synthesizing classical hierarchies, non-uniform models, and modern lower bounds, we recast class separation questions as matters of certificate efficiency. This perspective clarifies conjectured separations like \(\np \subsetneq \exptime\) and invites renewed analysis of how verifier certificates scale—a synthesis absent in prior literature.  

\section{Preliminaries}
\label{ref:preliminaries}

In this section, we establish the formal framework necessary for analyzing the relationship between certificate size and computational speed-up. We begin by defining the computational models under consideration and subsequently we introduce the remainder of the notation and definitions that will be used throughout the paper.

\subsection{Defining Solvers and Verifiers}

We consider languages over the finite alphabet \(\Sigma = \{0,1\}\). For a string \(x \in \Sigma^*\), let \(|x|\) denote its length. We use \(\Sigma^*\) to denote the set of all finite-length strings over \(\Sigma\) (including the empty string \(\varepsilon\)), and \(\Sigma^+\) to denote the set of nonempty strings over \(\Sigma\), i.e.,  \(\Sigma^+ = \Sigma^* \setminus \{\varepsilon\}\).
A \emph{language} \(L\) is a subset of \(\Sigma^*\).

\begin{definition}[Certificate-Based Deterministic Verifier]  
A \emph{certificate-based verifier} \(V\) for a language \(L\) is a deterministic Turing machine that takes as input a pair \((x, w)\), where \(x \in \Sigma^*\) is the instance and \(w \in \{0,1\}^*\) is a certificate. For all \(x \in L\), there exists a certificate \(w\) such that \(V(x, w) = 1\) (acceptance). For all \(x \not\in L\) and all certificates \(w\), \(V(x, w) = 0\) (rejection).  
\end{definition}  

\begin{definition}[Certificate Length]  
For a verifier \(V\) and input length \(n\), the \emph{certificate length} \(b(n)\) is the maximum length of any certificate \(w\) required for inputs \(x \in L\) with \(|x| = n\).  
\end{definition}  

\begin{definition}[Deterministic Solver]  
A \emph{deterministic solver} \(S\) for \(L\) is a deterministic Turing machine that, on input \(x \in \Sigma^*\), outputs \(S(x) = 1\) if \(x \in L\) and \(S(x) = 0\) otherwise. Equivalently, \(S\) is a verifier with certificate length \(b(n) = 0\). Solvers are closed under complement: if \(S\) decides \(L\) in time \(T_S(n)\), then its complement \(L^c\) is decidable in \(T_S(n)\) by flipping \(S\)’s output.  
\end{definition}  

\begin{definition}[Time Complexity]  
For a Turing machine \(M\), the \emph{time complexity} \(T_M(n)\) is the maximum number of steps \(M\) takes over all inputs \(x\) with \(|x| = n\). All time bounds \(f(n), g(n)\) are assumed to be \emph{time-constructible} (i.e., computable in \(O(f(n))\) time).  
\end{definition}  

Verifiers and solvers are modeled as multi-tape deterministic Turing machines:  
\begin{itemize}  
    \item The input \(x\) resides on a read-only input tape.  
    \item The certificate \(w\) (if present) is on a separate read-only certificate tape.  
    \item Work tapes are used for computation.  
    \item A write-once output tape records acceptance/rejection.  
\end{itemize}  
Unless stated otherwise, we assume this multi-tape model. 

\subsection{Additional notation}

We use standard asymptotic notation:  
\begin{itemize}  
    \item \(f(n) = O(g(n))\): \(\exists c > 0, n_0 \in \mathbb{N}\) such that \(f(n) \leq c \cdot g(n)\) for all \(n \geq n_0\).  
    \item \(f(n) = \Omega(g(n))\): \(\exists c > 0, n_0 \in \mathbb{N}\) such that \(f(n) \geq c \cdot g(n)\) for all \(n \geq n_0\). 
    \item \(f(n) = \Theta(g(n))\) means \(f(n) = O(g(n))\) and \(f(n) = \Omega(g(n))\). 
\end{itemize}  
\noindent
Furthermore, for a function \( f(n) \), the notation \( \poly(n) \) denotes the set of all functions bounded by a polynomial in \( n \), i.e., 
    \[
    \poly(n) = \bigcup_{k \geq 1} O(n^k),
    \]
\noindent
where \( k \) is a constant. The shorthand \( \poly \) (without arguments) represents an unspecified polynomial function.

Throughout the text we explicitly consider the following complexity classes:
\begin{itemize}  
    \item \(\p\): Languages decidable by deterministic Turing machines in polynomial time:  
    \[  
    \p = \bigcup_{k \geq 1} \mathsf{DTIME}(n^k),  
    \]  
    where \(\mathsf{DTIME}(f(n))\) denotes the class of languages decidable in \(O(f(n))\) time.  
      
    \item \(\np\): Languages verifiable in polynomial time with polynomial-length certificates:  
    \[  
    L \in \np \iff \exists \text{ verifier } V \text{ with } T_V(n) = \poly(n) \text{ and } b(n) = \poly(n).  
    \]  
      
    \item \(\exptime\): Languages decidable in exponential time:  
    \[  
    \exptime = \bigcup_{k \geq 1} \mathsf{DTIME}\left(2^{n^k}\right).  
    \]  
\end{itemize}  

\subsection{Language-Level time bounds}
\label{subsec:language}

In the statement of Theorem~\ref{thm:tradeoff} we discuss verification times of languages. In the remainder, such statements are intended to refer to inherent properties of the language, not to any specific algorithm:

 \begin{definition}[Language‐Level Time Bounds]
 \label{def:language-level}
Let \(L\subseteq\Sigma^*\) be a language, and \(b(n)\) a function bounding certificate length.
\begin{itemize}
  \item We say \emph{\(L\) has verification time \(\Omega(f(n))\) with \(b(n)\)-bit certificates} if \emph{every} deterministic, certificate‐based verifier for \(L\) that uses at most \(b(n)\) bits must run in \(\Omega(f(n))\) time.
  \item We say \emph{\(L\) has verification time \(O(g(n))\) with \(b(n)\)-bit certificates} if there \emph{exists} a deterministic, certificate‐based verifier for \(L\) that uses at most \(b(n)\) bits and runs in \(O(g(n))\) time.
\end{itemize}
\end{definition}

\subsection{Fundamental Relationships}

A verifier \(V\) with certificate length \(b(n)\) and verification time \(T_V(n)\) can be simulated by a solver enumerating all \(2^{b(n)}\) certificates, yielding a time complexity of \(O(2^{b(n)} \cdot T_V(n))\). This exponential overhead raises a central question: what is the minimal certificate length \(b(n)\) required to achieve a given speed-up over deterministic solvers? The \hyperref[thm:tradeoff]{Verifier Trade-off Theorem} answers this by establishing a tight trade-off between certificate size and computational advantage.

\section{The Verifier Trade-Off Hierarchy}  
\label{sec:hierarchy}  

We formalize the relationship between certificate length and verification efficiency. We re-emphasize that our focus is on language level time bounds, as defined in section \ref{subsec:language}. The following theorem establishes a fundamental trade-off: any reduction in language-level verification time necessitates a proportional increase in certificate length.  

\begin{theorem}[Verifier Trade-off Theorem]
\label{thm:tradeoff}
Let \( L \subseteq \Sigma^* \) be a language with two verifiers:  
\begin{enumerate}  
    \item \( V_1 \): Uses at most \( b_1(n) \)-bit certificates and has (language-level) verification time \( \Omega(f(n)) \) (Def.~\ref{def:language-level}).  
     % \item Every verifier for \(L\) using at most \(b_1(n)\) bits has (language‐level) verification time \(\Omega(f(n))\) (Def.~\ref{def:language-level})
    \item \( V_2 \): Uses at most \( b_2(n) \)-bit certificates (where \( b_2(n) \geq b_1(n) \)) and has verification time \( O(g(n)) \).  
    %   \item There exists a verifier for \(L\) using at most \(b_2(n)\) bits that runs in \(O(g(n))\) time (Def.~\ref{def:language-level}).
\end{enumerate}  
If \( f(n) \geq c \cdot g(n) \) for some constant \( c > 1 \) and all sufficiently large \( n \), then:  
\[
b_2(n) - b_1(n) = \Omega\left(\log\frac{f(n)}{g(n)}\right).
\]  
\end{theorem}  
\noindent
Hence, to achieve a speed-up factor of at least \( \Omega\left(\frac{f(n)}{g(n)}\right) \) in verification time (reducing from \( \Omega(f(n)) \) to \( O(g(n)) \)), the certificate length must increase by \( \Omega\left(\log\frac{f(n)}{g(n)}\right) \) bits. This quantifies the inherent trade-off between computational acceleration and certificate size. 

\begin{proof}  
Let \( L \subseteq \Sigma^* \) be a language with verifiers \( V_1 \) and \( V_2 \) as described. Let \( \Delta(n) = b_2(n) - b_1(n) \), and assume \( f(n) \geq c \cdot g(n) \) for some constant \( c > 1 \).  

For any input \( x \in L \), \( V_1 \) requires at least one valid certificate \( w_1 \in \{0,1\}^{b_1(n)} \). Since \( V_2 \) uses longer certificates, we decompose its certificates as \( w_2 = (w_1, d) \), where \( d \in \{0,1\}^{\Delta(n)} \) represents the additional bits.  To compare the verifiers’ efficiencies, we construct a deterministic solver \( S \) for \( L \):  
\begin{itemize}  
    \item On input \( x \), \( S \) fixes a minimal valid certificate \( w_1 \) for \( V_1 \).  
    \item \( S \) enumerates all \( 2^{\Delta(n)} \) possible extensions \( d \in \{0,1\}^{\Delta(n)} \), forming candidate certificates \( w_2 = (w_1, d) \).  
    \item For each \( w_2 \), \( S \) simulates \( V_2(x, w_2) \) for \( O(g(n)) \) steps.  
    \item \( S \) accepts \( x \) if any \( w_2 \) causes \( V_2 \) to accept.  
\end{itemize}  
\noindent
The solver \( S \) runs \( V_2 \) \( 2^{\Delta(n)} \) times, each taking \( O(g(n)) \) steps. Thus, \( S \)’s time complexity is:  
\[
T_S(n) = O\left(2^{\Delta(n)} \cdot g(n)\right).
\]  
However, since \( S \) is a deterministic solver for \( L \), it must take at least \( \Omega(f(n)) \) time to decide \( L \)—otherwise, \( V_1 \)’s language-level lower bound \( \Omega(f(n)) \) would be violated. Combining these:  
\[
O\left(2^{\Delta(n)} \cdot g(n)\right) \geq \Omega(f(n)).
\]  
Rearranging gives constants \( c_1, c_2 > 0 \) such that:  
\[
c_1 \cdot 2^{\Delta(n)} \cdot g(n) \geq c_2 \cdot f(n) \implies \Delta(n) \geq \log_2\left(\frac{c_2}{c_1} \cdot \frac{f(n)}{g(n)}\right).
\]  
This simplifies to \( \Delta(n) = \Omega\left(\log\frac{f(n)}{g(n)}\right) \), as required.  
\end{proof}  

\begin{remark}
The logarithmic dependence on \( f(n)/g(n) \) is asymptotically optimal. To see why, consider the information-theoretic and computational necessity of this scaling. First, by Shannon’s source coding theorem \citep{shannon1948mathematical}, distinguishing \( \Theta(f(n)/g(n)) \) distinct computational paths (each corresponding to a unique verification trajectory) requires at least \( \log_2(f(n)/g(n)) \) bits. Second, each additional certificate bit halves the search space for \( V_2 \), meaning \( \Delta(n) \) bits reduce the verification time by a \( 2^{\Delta(n)} \)-factor. If \( \Delta(n) \) were smaller than \( \Omega(\log(f(n)/g(n))) \), the solver \( S \) could contradict \( V_1 \)’s time lower bound by exploiting insufficient branching to cover all necessary paths. Thus, the bound is unavoidable.  
\end{remark}

\begin{corollary}[Speedup Upper Bound]
\label{cor:speedup-bound}
As before, let \( L \subseteq \Sigma^* \) be a language with two verifiers:
\begin{enumerate}
    \item \( V_1 \): Uses \( b_1(n) \)-bit certificates and (language-level) verification time \( \Omega(f(n)) \).
    \item \( V_2 \): Uses \( b_2(n) \)-bit certificates (\( b_2(n) \geq b_1(n) \)) and verification time \( O(g(n)) \).
\end{enumerate}
Then the verification time of \( V_2 \) is bounded by:
\[
g(n) = \Omega\left(\frac{f(n)}{2^{c \cdot (b_2(n) - b_1(n)}}\right),
\]
for some constant \( c > 0 \). Equivalently, the maximum possible speedup from increasing certificates by \( \delta(n) = b_2(n) - b_1(n) \) bits is at most exponential in \( \delta(n) \):
\[
\frac{f(n)}{g(n)} = O\left(2^{c \cdot \delta(n)}\right).
\]
\end{corollary}

\begin{proof}
From Theorem~\ref{thm:tradeoff}, there exists a constant \( c > 0 \) such that:
\[
b_2(n) - b_1(n) \geq c \cdot \log\frac{f(n)}{g(n)}.
\]
Exponentiating both sides:
\[
2^{b_2(n) - b_1(n)} \geq 2^{c \cdot \log\frac{f(n)}{g(n)}} = \left(\frac{f(n)}{g(n)}\right)^c.
\]
Rearranging for \( g(n) \):
\[
g(n) \geq \frac{f(n)}{2^{(b_2(n) - b_1(n))/c}} = \Omega\left(\frac{f(n)}{2^{c \cdot (b_2(n) - b_1(n))}}\right).
\]
The speedup factor satisfies:
\[
\frac{f(n)}{g(n)} \leq 2^{c \cdot (b_2(n) - b_1(n))} = O\left(2^{c \cdot \delta(n)}\right).
\]
\end{proof}

\begin{remark}[Model-Dependent Speedup Limitations]
\label{rem:model-limits}
The speedup upper bound in Corollary~\ref{cor:speedup-bound} assumes an optimal computational model (e.g., multi-tape Turing machines). In some case, for example in the case of more restrictive models like \emph{single-tape Turing machines}, inherent overheads (e.g., linear-time head movements) may prevent achieving the theoretical speedup \citep[see also,][]{arora2009computational}. For example, in Section~\ref{subsec:natural}, we show that for natural problems like \textsc{Periodic} and \textsc{Rotation} the theoretical bounds are indeed not achieved due to model-induced overheads. This highlights the distinction between \emph{information-theoretic} speedup (governed by certificate size) and \emph{computational} speedup (limited by model constraints).
\end{remark}

The derivation above permits a result directly relating deterministic solvers to verifiers. Corollary \ref{cor:solververifier}, which is a special case of the previous result where $b_1(n) =0$ formalizes this relationship. 

\begin{corollary}[Solver-Verifier Trade-off Corollary]
\label{cor:solververifier}
Let \( L \subseteq \Sigma^* \) be a language decidable by a deterministic solver \( S \) in (language-level) time \( \Omega(f(n)) \). If \( L \) has a verifier \( V \) with certificate length \( b(n) \) and verification time \( O(g(n)) \), then:
\[
b(n) = \Omega\left(\log\frac{f(n)}{g(n)}\right).
\]
\end{corollary}

\begin{proof}
This follows directly from \autoref{thm:tradeoff} by setting \( b_1(n) = 0 \).  
\end{proof}

Corollary \ref{cor:solververifier} permits an independent proof that is potentially instructive towards the proof of Theorem \ref{thm:tradeoff}.

\begin{proof}[Alternative Proof (Stand-Alone)]
Construct a deterministic solver \( S \) for \( L \) as follows:
\begin{itemize}
    \item \textbf{Input}: \( x \) of length \( n \).
    \item \textbf{Enumerate}: All \( 2^{b(n)} \) certificates \( c \in \{0,1\}^{b(n)} \).
    \item \textbf{Simulate}: For each \( c \), run \( V(x, c) \) for \( O(g(n)) \) steps.
    \item \textbf{Decide}: Accept \( x \) iff \( V \) accepts any \( c \).
\end{itemize}
\noindent
The total runtime of \( S \) is:
\[
O\left(2^{b(n)} \cdot g(n)\right).
\]
\noindent
Since \( S \) decides \( L \) and \( S \)’s runtime is at least \( \Omega(f(n)) \), there exist constants \( k_1, k_2 > 0 \) such that:
\[
k_1 \cdot 2^{b(n)} \cdot g(n) \geq k_2 \cdot f(n).
\]
\noindent
Rearranging and taking base-2 logarithms:
\[
b(n) \geq \log_2\left(\frac{k_2}{k_1} \cdot \frac{f(n)}{g(n)}\right) = \Omega\left(\log\frac{f(n)}{g(n)}\right).
\]
\end{proof}

%% -- APPLICATIONS %%

\section{Applications and further examples}
\label{sec:applications}

This section explores several interpretations and applications of our main result, the \hyperref[thm:tradeoff]{Verifier Trade-off Theorem}, and its important consequence for comparing solver and verifier complexities, the \hyperref[cor:solververifier]{Solver-Verifier Trade-off Corollary}. First, we provide a general discussion on using certificates to accelerate computation. Then, we analyze the implications of our theorem for natural problems and complexity class separations.

\subsection{Certificate Size as a Computational Resource}  
\label{subsec:certresource}  

The \hyperref[thm:tradeoff]{Verifier Trade-Off Theorem} establishes a relationship between certificate size and verification time: reducing computation effectively requires investing bits in the certificate.  To illustrate, consider two verifiers for a language \( L \):  
\begin{itemize}  
    \item \( V_1 \): A baseline verifier with no certificate (\( b_1(n) = 0 \)) and language-level verification time \( f(n) \).  
    \item \( V_2 \): An optimized verifier with certificate length \( b_2(n) = \delta(n) \) and verification time \( g(n) \).  
\end{itemize}  
\noindent
By the theorem, the certificate size difference \( \delta(n) = b_2(n) - b_1(n) \) satisfies:  
\[
\delta(n) = \Omega\left(\log\frac{f(n)}{g(n)}\right).
\]  
Rewriting, this implies that the achievable \emph{speed-up factor} \( S(n) = \frac{f(n)}{g(n)} \) is upper-bounded by:  
\[
S(n) = O\left(2^{\delta(n)}\right).
\]  
In plain terms: each additional certificate bit can at most double the verification speed. Conversely, halving the time requires investing at least one more bit.

To visualize the trade-off, fix the input size \( n \) so that the baseline verifier’s time is a constant:  
\[
f(n) = 1024.
\]  
We now examine how verification time \( g(n) \) decreases as the certificate size \( \delta(n) \) increases. For fixed \( n \), the trade-off becomes:  
\[
g(n) = O\left(\frac{1024}{2^{\delta(n)}}\right).
\]  
\noindent
Figure \ref{fig:cert-tradeoff} plots this exponential upper bound (blue curve) and achievable values for integer \( \delta(n) \) (red dots). Figure \ref{fig:cert-tradeoff} demonstrates the exponential decrease in verification time as certificate size increases. This highlights that each additional certificate bit can halve the verification time, making certificates a powerful but costly resource. However, it's important to note that verification time cannot decrease below a constant, and certificate length is limited by the computational model, even though the blue curve visually approaches \( g(n) = 0 \) as \( \delta(n) \) increases.

\begin{figure}[ht]  
\centering  
\begin{tikzpicture}  
\begin{axis}[  
    xlabel={Certificate Bits (\( \delta(n) \))},  
    ylabel={Verification Time \( g(n) \)},  
    xmin=0, xmax=10,  
    ymin=1, ymax=1024,  
    axis lines=left,  
    grid=both,  
    title={Exponential Speed-Up via Certificate Bits (Fixed \( n \))},  
    legend pos=north east,  
    ymode=log,  
    log ticks with fixed point,  
]  
% Upper bound curve
\addplot[domain=0:10, samples=11, thick, blue] {1024 / 2^x};  

% Discrete achievable values
\addplot[only marks, mark=*, mark size=2pt, red] coordinates {  
    (1, 512)  
    (2, 256)  
    (3, 128)  
    (4, 64)  
    (5, 32)  
    (6, 16)  
    (7, 8)  
    (8, 4)  
    (9, 2)  
    (10, 1)  
};  
\end{axis}  
\end{tikzpicture}  
\caption{  
    For fixed input size \( n \), each additional certificate bit (at most) halves the verification time. The blue curve shows the theoretical upper bound from the  \hyperref[thm:tradeoff]{Verifier Trade-Off Theorem}; red dots mark achievable values at integer \( \delta(n) \).  
}  
\label{fig:cert-tradeoff}  
\end{figure}
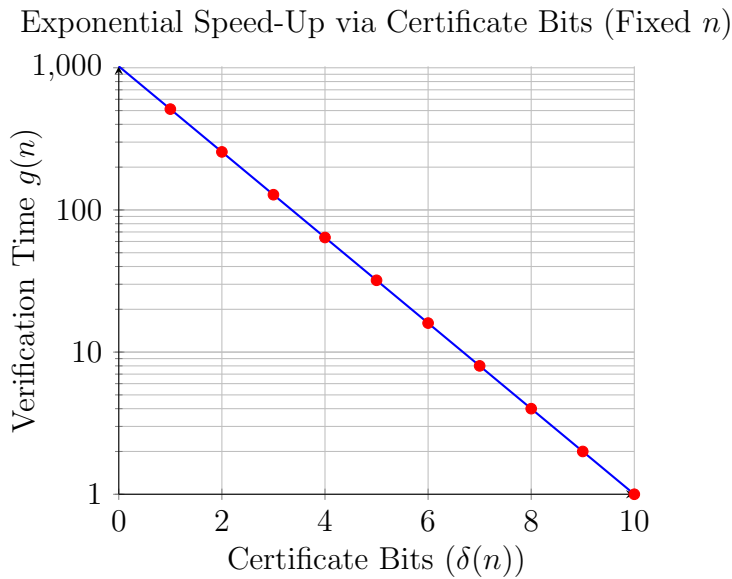  

Returning to general input size \( n \), the trade-off implies:
\[
\delta(n) \geq \log_2\left(\frac{f(n)}{g(n)}\right),
\]
which bounds the \emph{minimum certificate length} needed to achieve a desired speed-up. Table \ref{tab:cert-speed-tradeoff} summarizes the asymptotic relationship between certificate size and the resulting speed-up, providing concrete examples. These examples further illustrate that certificates act as a form of computational currency: each bit you add lets you buy a halving of the work the verifier must do. But the currency is expensive — logarithmic speed-ups are relatively cheap, while exponential savings require linear-sized certificates or more. This quantifies an intuitive trade-off in real-world algorithm design.

\begin{table}[ht]
\centering
\begin{tabular}{|c|c|c|}
\hline
% \textbf{Certificate Size} (\(\delta(n)\)) & \textbf{Speed-Up} (\(S(n) = \frac{f(n)}{g(n)}\)) & \textbf{Example} \\
\textbf{Certificate Size} & \textbf{Speed-Up}  & \textbf{Example} \\
\hline
\(\Omega(\log n)\) & Polynomial (\( O(n^c) \)) & \(f(n) = n^2, \ g(n) = n\) \\ 
\hline
\(\Omega(n)\) & Exponential (\( O(2^n) \)) & \(f(n) = 2^n, \ g(n) = \poly(n)\) \\ 
\hline
\(\Omega(n^k)\) & Super-exponential & \(f(n) = 2^{n^k}, \ g(n) = \poly(n)\) \\ 
\hline
\end{tabular}
\caption{Certificate-Speed Trade-Off (Asymptotic Cases). This table illustrates the relationship between certificate size  (\(\delta(n)\)) and the achievable speed-up in verification time (\(S(n) = \frac{f(n)}{g(n)}\)). It shows, for a solver with time complexity \(O(f(n))\), the speed-up factor for a verifier with time complexity \(g(n)\) when provided with a certificate of a given size.}
\label{tab:cert-speed-tradeoff}
\end{table}

\subsection{Natural Problems}
\label{subsec:natural}

In this section we explore the implications of the \hyperref[thm:tradeoff]{Verifier Trade-Off Theorem} for two well-studied natural problems: \period and \rotate. These problems, fundamental to string processing and algorithm design \citep[see, e.g.,][]{cormen2009introduction}), demonstrate how certificate complexity captures the trade-off between computation and information. 
%In both of these applied examples, to motivate the solver's complexity time lower bounds, we consider single-tape Turing machines.

\begin{example}[The \period Problem]
\label{ex:periodic}
Let \(\Sigma\) be a finite alphabet. The language \(\textsc{Periodic} \subseteq \Sigma^*\) consists of all strings \(x\) that are repetitions of a smaller period \(p\), i.e.,
\[
\textsc{Periodic} = \left\{ x \in \Sigma^* \mid \exists p \in \Sigma^+,\ k \geq 2 \text{ such that } x = p^k \right\}.
\]
On a multi-tape Turing machine (TM), a deterministic solver for \period might need to find the period. A naive approach of checking all possible period lengths $\ell$ from 1 to $n/2$ and verifying would take $O(n^2)$ time (since verifying each candidate period $\ell$ requires $O(n)$ time),\footnote{In fact, one can solve \textsc{Periodic} in $O(n\log n)$ time (or even $O(n)$ on a RAM) by computing the failure function of $x$ (Knuth–Morris–Pratt) \citep{cormen2009introduction}.} so we proceed under this assumption for our illustration. More efficient solvers might exist, but it's plausible that solving might take more than linear time without the period given and we proceed under this assumption for the sake of our discussion.

A verifier \(V\) for \period on a multi-tape TM can achieve \(O(n)\) time with a \(\Theta(\log n)\)-bit certificate (the period length \(\ell = |p|\)). The verifier:
\begin{enumerate}
\item Reads $\ell$ from the certificate.
\item Checks if $n \bmod \ell = 0$.
\item Verifies if $x[i] = x[i \bmod \ell]$ for all $i$ from 0 to $n-1$. This takes $O(n)$ time on a multi-tape TM by keeping one head on index $i$ and another on $i \bmod \ell$.
\end{enumerate}
\noindent
Here, the verifier achieves a time complexity of $O(n)$, which is likely faster than the time required for a solver to find the period from scratch on a multi-tape TM (which could be $O(n^2)$ with a simple approach or possibly better with more advanced algorithms, but still potentially super-linear in the context of just verification being linear).
\end{example}

The \period problem illustrates how certificate size captures the inherent trade-off between brute-force search and directed verification. Unlike the solver, which must exhaustively test \(O(n)\) periods, the verifier leverages a logarithmic-sized hint to avoid quadratic work. This aligns with the intuition that computational shortcuts via certificates are exponentially priced in advice size, as formalized by the theorem.

\begin{example}[The \rotate Problem]  
\label{ex:rotation}  
Let \(\Sigma\) be a finite alphabet. The language \(\rotate \subseteq \Sigma^* \times \Sigma^*\) consists of pairs \((A, B)\) where \(B\) is a cyclic rotation of \(A\). Formally:  
\[
\textsc{Rotation} = \left\{ (A, B) \mid \exists k \in [0, |A|) \text{ such that } B = A[k:] \cdot A[:k] \right\}.
\] 
\noindent
To illustrate, consider the following example using ASCII characters (understood to be encoded as binary strings under a standard encoding): the pair \((\texttt{abcde}, \texttt{cdeab})\) belongs to \(\rotate\) because 'cdeab' is a rotation of 'abcde'.

A deterministic solver \(S\) for \rotate on a multi-tape TM using again a \emph{naïve approach}, will:
\begin{enumerate}  
	\item Check if \(|A| = |B|\).  
	\item For each \(k \in [0, n-1]\):  
	\begin{itemize}
		\item Generate the candidate rotation \(B' = A[k:] \cdot A[:k]\).  
		\item Compare \(B'\) to \(B\) character-by-character.  
	\end{itemize}
\end{enumerate}
\noindent
Each rotation check takes \(O(n)\) time, and there are \(n\) rotations, hence the total time is:  
\[
T_S(n) = O(n^2).\footnote{Alternatively, one can decide \textsc{Rotation} in $O(n)$ time by checking if $B$ is a substring of $AA$, e.g., via KMP \citep{cormen2009introduction}.}
\]  

By contrast, a verifier \(V\) uses a \(\Theta(\log n)\)-bit certificate (the rotation index \(k\)):  
\begin{enumerate}  
	\item Read \(k\) (encoded in \(\lceil \log_2 n \rceil\) bits).  
	\item Verify \(|A| = |B| = n\).  
	\item For each \(i \in [0, n-1]\):  
	\begin{itemize}
		\item Compute \(j = (i + k) \bmod n\).  
		\item Check \(B[i] = A[j]\).  
	\end{itemize}
\end{enumerate}
\noindent
On a multi-tape TM, step 3 above uses two tape heads: one for \(B\) (fixed at position \(i\)) and one for \(A\) (moving to \(j\)). Additionally, each character comparison takes \(O(1)\) time giving a total time of:
\[
T_V(n) = O(n).
\]  
\noindent
By the Verifier Trade-off Theorem:  
\[
b(n) = \Omega\left(\log\frac{f(n)}{g(n)}\right) = \Omega\left(\log\frac{n^2}{n}\right) = \Omega(\log n).
\]  
\noindent
The certificate size \(\Theta(\log n)\) matches this bound.  

To conclude, the solver’s brute-force search requires \(O(n^2)\) time, while the verifier leverages the certificate to reduce the problem to direct comparisons, achieving \(O(n)\) time. This demonstrates the theorem’s core principle: certificates enable exponential speed-ups by compressing search effort into logarithmic advice. 
\end{example}

\subsection{Complexity Class Separations}
\label{subsec:classsep}

The \hyperref[thm:main]{Verifier Hierarchy Theorem} offers a novel framework for analyzing complexity class separations. We first demonstrate this idea by re-examining the relation between \np and \exptime \citep{arora2009computational, Papadimitriou1994}, providing a new perspective based on certificate-size trade-offs. We then explore the potential implications of the theorem for the \( \p \) vs. \( \np \) question. Finally, we discuss a conditional analysis of \SAT using the theorem.

\begin{example}[Certificate Size Lower Bound for \(\exptime\)-Complete Languages]
\label{ex:exptime}
Let \( L \) be an \(\exptime\)-complete language. Then:
\begin{itemize}
    \item \( L \in \exptime \): A deterministic solver \( S \) decides \( L \) in time \( f(n) = O(2^{p(n)}) \), for some polynomial \( p(n) \) of degree \( k \ge 1 \).
    \item \( L \) is \(\exptime\)-hard: All problems in \(\exptime\) reduce to \( L \) in polynomial time.
\end{itemize}
\noindent
Suppose \( L \in \np \). Then \( L \) has a polynomial-time verifier \( V \) with certificate length \( b(n) = \poly(n) \) and verification time \( g(n) = \poly(n) \). By the Solver-Verifier Trade-off Corollary:
\[
b(n) = \Omega\left(\log\frac{f(n)}{g(n)}\right) = \Omega\left(\log\frac{2^{p(n)}}{\poly(n)}\right) = \Omega(p(n) - \log \poly(n)) = \Omega(p(n)).
\]

\noindent
Thus, any verifier for \( L \) must use certificates of length at least \( \Omega(p(n)) \), where \( p(n) \) is the exponent in the exponential-time solver. This implies that for certain \(\exptime\)-complete languages, the verifier must read certificates longer than any fixed-degree polynomial — unless \(\np\) admits arbitrarily large polynomial certificate bounds.

While this does not give a formal contradiction, it shows that \(\exptime\)-complete languages require verifiers with unusually long certificates, distinguishing them structurally from typical \(\np\) languages. %Therefore, under standard assumptions, we obtain evidence that \( \np \ne \exptime \).
\end{example}

\begin{remark}[On \(\np \neq \exptime\)]
\label{rem:np-exptime}
Our analysis does not resolve the open problem \(\np \neq \exptime\)\footnote{Whether \(\np \subsetneq \exptime\) holds unconditionally remains a major open question in structural complexity; see, e.g., \citep{sipser2013, arora2009computational}.}. Instead, it demonstrates that if \exptime-complete problems \emph{require} super-polynomial certificates for polynomial-time verification—consistent with the conjecture \(\exptime \not\subseteq \np\)—then the Verifier Trade-off Theorem provides a novel lens through which to view this separation.  
%This aligns with classical hierarchy theorems (e.g., \(\p \neq \exptime\)) and underscores the role of certificate complexity in distinguishing nondeterministic and deterministic classes.
\end{remark}

The \hyperref[thm:main]{Verifier Hierarchy Theorem} reveals a bidirectional connection between certificate size and computational hardness, suggesting two potential implications for investigating the \( \p \) vs. \( \np \) problem: 

\begin{example}[Implication 1: Certificate Efficiency and potential Complexity Collapse] 
Suppose there exists a polynomial-time verifier for an $\np$-complete language $L$ with certificate length $b(n)$. We analyze two regimes:
\begin{itemize}
    \item \textbf{Constant certificates ($b(n) = O(1)$}: 
    By the \hyperref[cor:solververifier]{Solver-Verifier Trade-off Corollary}, this implies:
    \[
    O(1) = \Omega\left(\log\frac{f(n)}{\poly(n)}\right) \implies f(n) = \poly(n).
    \]
    Thus, $L$ has a polynomial-time deterministic solver, implying $\p = \np$.

    \item \textbf{Sub-linear certificates ($b(n) = O(n^e)$ for $0 < e < 1)$}: 
    The corollary yields:
    \[
    O(n^e) = \Omega\left(\log\frac{f(n)}{\poly(n)}\right) \implies f(n) = 2^{O(n^e)} \cdot \poly(n).
    \]
    While the latter does not directly imply $\p = \np$, it contradicts the \emph{Exponential Time Hypothesis (ETH)} \citep{impagliazzo2001complexity}, which posits that $\mathsf{3\text{-}SAT}$ requires $2^{\Omega(n)}$ time. Sub-exponential solvers ($2^{o(n)}$) for $\np$-complete problems are incompatible with ETH. 
\end{itemize}
\noindent
Hence, the existence of constant certificates collapses $\p$ and $\np$, while sub-linear certificates would violate ETH. Both outcomes would challenge foundational conjectures, demonstrating that non-trivial certificate compression implies profound computational savings.
\end{example}

\begin{example}[Implication 2: Demonstrating \(\p \neq \np \) using certificate size]  
\label{ex:pathway2}
Let \(L\) be an \(\np\)-complete language. Suppose \emph{every} polynomial-time verifier for \(L\) requires certificates of size \(b(n) = \Omega(n)\). By the \hyperref[cor:solververifier]{Solver-Verifier Trade-off Corollary}, any deterministic solver for \(L\) must satisfy:  
\[
\Omega(n) = \Omega\left(\log\frac{f(n)}{\poly(n)}\right) \implies f(n) = 2^{\Omega(n)} \cdot \poly(n) = 2^{\Omega(n)}.
\]  
This implies \(L\) cannot be decided in polynomial time. Since \(L\) is \(\np\)-complete, this would prove \(\p \neq \np\).  
\end{example}

The examples above reinterpret \( \mathsf{P} \) vs. \( \mathsf{NP} \) through the lens of \emph{certificate efficiency}. Specifically, the existence of sub-linear certificates would imply that nondeterminism is relatively "inexpensive," potentially leading to the collapse of the hierarchy (\( \p = \np \)). Conversely, linear certificates would indicate that nondeterminism is "costly," preserving the separation.

These examples are not purely theoretical. The following example demonstrates that, under the assumption that \( \SAT \not\in \p \), \( \SAT \) requires linear certificates. The \( \SAT  \) problem is a cornerstone of computational complexity, serving as both a canonical \( \np \)-complete problem \citep{cook2023complexity} and a model for investigating phase transitions in computational hardness \citep{mertens2006number}. Assuming the widely held conjecture that \( \SAT \) is not solvable in deterministic polynomial time (i.e., \( \SAT \not\in \p \)). the next example demonstrates how the \hyperref[thm:tradeoff]{Verifier Trade-off Theorem} formalizes the relationship between certificate size and computational hardness. Specifically, the example shows that if \SAT is hard, its hardness is inherently tied to the \emph{certificate inefficiency} of \np)-complete problems.

\begin{example}[Certificate Complexity of \SAT]
\label{ex:sat-hardness}

We analyze the certificate complexity of \SAT under the widely held assumption that \( \SAT \not\in \p \). This assumption implies that any deterministic solver for \SAT requires super-polynomial time. We can use the \hyperref[cor:solver-verifier]{Solver-Verifier Trade-off Corollary} (Corollary 4.3) to determine the minimum certificate size required for a polynomial-time verifier. The high entropy of solutions for certain \SAT instances provides intuition for why large certificates might be necessary.

\begin{definition}[Adversarial \SAT Instance]
\label{def:adversarial}
For \(n \geq 1\), consider instances \(I_n\) of \SAT{} on \(n\) variables constructed adversarially. For example, instances near the satisfiability phase transition threshold (\(m/n \approx 4.26\)) \citep{selman1996generating} are known to be computationally hard and often possess a complex solution space.
\end{definition}

\begin{lemma}[Conditional Entropy Intuition]
\label{lem:entropy-intuition}
\emph{(Heuristic.)}  For hard \SAT instances \(I_n\), such as those constructed adversarially near the phase transition threshold, the entropy \(H(I_n)\) of the set of satisfying assignments is expected to be linear, i.e., \(H(I_n) = \Omega(n)\).\footnote{This heuristic is motivated by empirical studies of random 3-SAT at the critical ratio, which report \(\log_2 \mathbb{E}[N]\approx0.18n\) \citep{mertens2006number}.}
\end{lemma}

\begin{proof}[Proof Sketch]
Based on calculations for random \SAT at the critical threshold \(\alpha \approx 4.26\), the expected number of satisfying assignments \(\mathbb{E}[N]\) suggests an entropy lower bound \(\log_2 \mathbb{E}[N] \approx 0.182n = \Omega(n)\). While this applies to random ensembles, it provides strong intuition that specifically constructed hard instances can maintain high solution entropy. Intuitively, if \( \SAT \not\in \p \), no polynomial-time preprocessing could significantly compress this entropy on all instances without aiding a solver.
\end{proof}

Now, we apply the main result of this paper:

\begin{theorem}[Conditional Certificate Lower Bound via Trade-off Theorem]
\label{thm:sat-cert-lower}
Assuming \( \SAT \not\in \p \), any polynomial-time verifier \(V\) for \SAT requires certificates of length \(b(n) = \Omega(n)\).
\end{theorem}

\begin{proof}
Let \(V\) be a polynomial-time verifier for \SAT. This means its verification time \(T_V(n) = g(n) = \text{poly}(n)\), and it uses certificates of length \(b(n)\).
The assumption \( \SAT \not\in \p \) implies that any deterministic solver \(S\) for \SAT requires super-polynomial time. Let the time complexity of the fastest possible solver be \(T_S(n) = f(n)\). Since \(\SAT \in \np \subseteq \exptime\), we know \(f(n)\) is at most exponential, but by assumption \(f(n)\) is super-polynomial. For instance, under the Exponential Time Hypothesis (ETH), \(f(n) = 2^{\Omega(n)}\). Let's assume conservatively that \(f(n) = 2^{\Omega(n^\epsilon)}\) for some \(\epsilon > 0\).

Applying the \hyperref[cor:solver-verifier]{Solver-Verifier Trade-off Corollary} (Corollary 4.3), which states \(b(n) = \Omega\left(\log \frac{f(n)}{g(n)}\right)\):
\begin{align*}
b(n) &= \Omega\left(\log \frac{f(n)}{g(n)}\right) \\
    &= \Omega\left(\log \frac{2^{\Omega(n^\epsilon)}}{\text{poly}(n)}\right) \\
    &= \Omega\left( \Omega(n^\epsilon) - \log(\text{poly}(n)) \right) \\
    &= \Omega(n^\epsilon).
\end{align*}
If we assume ETH (\(f(n)=2^{\Omega(n)}\)), the result is stronger:
\begin{align*}
b(n) &= \Omega\left(\log \frac{2^{\Omega(n)}}{\text{poly}(n)}\right) \\
    &= \Omega\left( \Omega(n) - O(\log n) \right) \\
    &= \Omega(n).
\end{align*}
Thus, under standard hardness assumptions for \SAT (like ETH implying \(f(n)\) is exponential), a polynomial-time verifier necessitates certificates of linear length. The entropy considerations in Lemma~\ref{lem:entropy-intuition} align with this conclusion, suggesting that the complexity quantified by the trade-off theorem reflects an underlying information-theoretic requirement for hard instances.
\end{proof}

\begin{corollary}[Conditional Separation \(\p \neq \np\)]
\label{cor:p-neq-np-from-sat}
If polynomial-time verifiers for \SAT require \(b(n) = \Omega(n)\) certificates (as implied by Theorem~\ref{thm:sat-cert-lower} under ETH), then \(\p \neq \np\).
\end{corollary}
\begin{proof}
Assume for contradiction that \(\p = \np\). Then \(\SAT \in \p\), meaning there exists a polynomial-time solver \(S\) for \SAT. This solver \(S\) can be viewed as a verifier with certificate length \(b(n)=0\). This contradicts the requirement that \(b(n) = \Omega(n)\) from Theorem~\ref{thm:sat-cert-lower}. Therefore, the initial assumption must be false, and \(\p \neq \np\).
\end{proof}
\end{example}

\begin{remark}[Conditional Nature of Results]
The results presented in the last example are conditional upon the assumption that \( \SAT \not\in \p \). Although these results do not provide an unconditional proof of \( \p \neq \np \), they effectively demonstrate how certificate complexity, viewed as a measure of nondeterministic computational effort, offers an interesting perspective for analyzing complexity class separations.
\end{remark}

In conclusion, our final example reframes the \( \p \) vs. \( \np \) question in terms of \emph{certificate efficiency}: if \( \np \)-complete problems necessitate polynomial-sized certificates, even under hardness assumptions, then nondeterminism cannot be reduced to determinism without incurring an exponential overhead. This aligns with the prevailing intuition that the essence of \( \p \neq \np \) lies in the inherent \emph{cost} of verification shortcuts. The \hyperref[thm:tradeoff]{Verifier Trade-Off Theorem} emphasizes that certificate size is not merely a syntactic attribute but a fundamental \emph{computational} resource. Specifically, for \SAT, the exponential disparity between solver and verifier runtimes (under the assumption that \( \p \neq \np \)) directly implies the requirement for linear certificates, as demonstrated in Theorem \ref{cor:p-neq-np-from-sat}. This exemplifies how certificate complexity provides a unified framework for reasoning about computational hardness across deterministic and nondeterministic models.

\section{Comments}  
\label{sec:comments}  

The Verifier Trade-off Theorem establishes a rigorous relationship between certificate size and verification time, but its implications and limitations invite deeper reflection. At its core, the theorem quantifies a fundamental tension: reducing verification time from \(T(n)\) to \(T(n)/2^{k}\) requires \(\Omega(k)\) additional certificate bits, as shown by the bound \(\Delta(n) = \Omega\left(\log \frac{f(n)}{g(n)}\right)\). This creates a hierarchy where computational shortcuts via certificates are exponentially priced in advice size, offering a novel lens to compare classical complexity classes. %For instance, our separation \(\mathsf{NP} \subsetneq \mathsf{EXPTIME}\) arises directly from the observation that exponential-time problems demand super-polynomial certificates for polynomial-time verification—a distinction invisible to traditional hierarchy theorems.  

However, several limitations constrain this framework. First, our results hold for worst-case complexity, whereas practical verification often exploits structured instances (e.g., SAT solvers leverage formula symmetry). Whether average-case certificate complexity aligns with worst-case bounds remains open. Second, the model assumes deterministic verification, excluding probabilistic or interactive paradigms like PCPs \citep{ben2008short}. Recent work extends probabilistically checkable proofs (PCPs) to real-number systems, demonstrating algebraic methods for verification in continuous domains \citep{baartse2017algebraic}. While our focus is on discrete certificates, such approaches highlight broader trade-offs between proof structure and verification efficiency. Extending the hierarchy to probabilistic verifiers, where randomness might reduce certificate size, poses a significant challenge. Third, conditional results (e.g., Theorem~\ref{cor:p-neq-np-from-sat}) depend on conjectures like \( \SAT \not\in \p \). While plausible, this leaves the framework contingent on unresolved assumptions.  

Central to future work is the question of tightness. For a function pair \(f(n) = 2^{2^n}\) and \(g(n) = O(1)\), our theorem mandates \(\Delta(n) = \Omega(2^n)\), but problem-specific structure (e.g., succinct representations of periodic strings \citep{jacobson1989space, he2010succinct}) might permit smaller certificates. Similarly, average-case analyses could refine our understanding: do typical instances of \rotate or \period require fewer bits than adversarial worst cases? Another direction involves probabilistic or quantum verification \citep{aaronson2023certified}. For example, quantum Merlin-Arthur (QMA) protocols \citep[see, e.g.,][]{marriott2005quantum} might achieve tighter trade-offs, compressing certificates further via quantum state encoding.  

Broader implications emerge at the intersection of complexity theory and applied computation. In parameterized complexity, kernelization compresses instances while preserving solvability—a process mirroring certificate generation \citep{downey2013fundamentals}. Our hierarchy’s certificate-time trade-off \(\left(b(n) = \Omega\left(\log \frac{f(n)}{g(n)}\right)\right)\) could inform kernel lower bounds, particularly for problems resistant to efficient preprocessing. Cryptographic proof systems, such as succinct non-interactive arguments (SNARGs) \citep{chiesa2014succinct}, also resonate with our framework: these systems prioritize minimal proof size for efficient verification, echoing the certificate-time duality.  

Ultimately, the hierarchy reinterprets computational hardness as a resource allocation problem. Just as time and space hierarchies delineate the cost of computation in their respective dimensions, certificate complexity formalizes the "price" of nondeterministic shortcuts. This reframing obviously does not resolve \p vs \np, but it potentially provides an interesting perspective: proving that \np-complete languages require \(\Omega(n)\)-bit certificates for polynomial-time verification would imply \( \p \neq \np \). %Conversely, sub-linear certificates would collapse the hierarchy, suggesting \(\p = \np \).  

% In this light, certificate complexity emerges not merely as a theoretical construct but as a measurable resource shaping computation’s intrinsic costs. Whether through adversarial entropy arguments or practical heuristics, the interplay of certificates and time promises to deepen our understanding of what makes problems hard—and how we might circumvent that hardness.  

\section{Declaration of generative AI and AI-assisted technologies in the writing process.}

During the preparation of this work the author(s) used several commercially available LLMs (including ChatGPT, Gemini, and Claude) in order to refine writing clarity, explore related literature, and verify proof concepts. After using these tools, the author(s) reviewed and edited the content as needed and take(s) full responsibility for the content of the published article.

\bibliographystyle{elsarticle-num}
\bibliography{references}

\end{document}